\documentclass{article}

\usepackage{microtype}
\usepackage{graphicx}
\usepackage{subfigure}
\usepackage{booktabs} 

\usepackage{adjustbox}

\usepackage{amsmath}  
\usepackage{xcolor}
\usepackage{amsthm}
\usepackage{xcolor,colortbl}
\usepackage{multirow}
\usepackage{empheq} 

\usepackage{hyperref}

\newtheorem{proposition}{Proposition}

\DeclareMathOperator*{\diag}{\mathrm{diag}}
\DeclareMathOperator*{\grad}{\mathrm{grad}}

\newcommand{\MM}{\mathcal{M}}
\newcommand{\hp}{\hphantom}
\newcommand{\adsn}{\mathrm{AdS_N}}
\newcommand{\boxedeq}[2]{\begin{empheq}[box={\fboxsep=6pt\fbox}]{align}\label{#1}#2\end{empheq}}

\usepackage{amsmath}  
\usepackage{amsfonts}

\usepackage[accepted]{icml2021}

\icmltitlerunning{Directed Graph Embeddings in Pseudo-Riemannian Manifolds}

\begin{document}

\twocolumn[
\icmltitle{Directed Graph Embeddings in Pseudo-Riemannian Manifolds}

\icmlsetsymbol{equal}{*}

\begin{icmlauthorlist}
\icmlauthor{Aaron Sim}{bai}
\icmlauthor{Maciej Wiatrak}{bai}
\icmlauthor{Angus Brayne}{bai}
\icmlauthor{P\'aid\'i Creed}{bai}
\icmlauthor{Saee Paliwal}{bai}
\end{icmlauthorlist}

\icmlaffiliation{bai}{BenevolentAI, London, UK}
\icmlcorrespondingauthor{Aaron Sim}{aaron.sim@benevolent.ai}

\icmlkeywords{Machine Learning, ICML, Embeddings, Riemannian manifolds}

\vskip 0.3in
]

\printAffiliationsAndNotice{} 
\begin{abstract}
The inductive biases of graph representation learning algorithms are often encoded in the background geometry of their embedding space. In this paper, we show that general directed graphs can be effectively represented by an embedding model that combines three components: a pseudo-Riemannian metric structure, a non-trivial global topology, and a unique likelihood function that explicitly incorporates a preferred direction in embedding space. We demonstrate the representational capabilities of this method by applying it to the task of link prediction on a series of synthetic and real directed graphs from natural language applications and biology. In particular, we show that low-dimensional cylindrical Minkowski and anti-de Sitter spacetimes can produce equal or better graph representations than curved Riemannian manifolds of higher dimensions. 
\end{abstract} 

\section{Introduction}
\label{introduction}
Representation learning of symbolic objects is a central area of focus in machine learning. Alongside the design of deep learning architectures and general learning algorithms, incorporating the right level of inductive biases is key to efficiently building faithful and generalisable entity and relational embeddings \cite{battaglia2018}.

In graph representation learning, the embedding space geometry itself encodes many such inductive biases, even in the simplest of spaces. For instance, vertices embedded as points in Euclidean manifolds, with inter-node distance guiding graph traversal and link probabilities \cite{grover2016, perozzi2014}, carry the underlying assumption of homophily with node similarity as a metric function. 

The growing recognition that Euclidean geometry lacks the flexibility to encode complex relationships on large graphs at tractable ranks, without loss of information \cite{nickel2011, bouchard2015}, has spawned numerous embedding models with non-Euclidean geometries. Examples range from complex manifolds for simultaneously encoding symmetric and anti-symmetric relations \cite{trouillon2016}, to statistical manifolds for representing uncertainty \cite{vilnis2015}. 

One key development was the introduction of hyperbolic embeddings for representation learning \cite{nickel2017, nickel2018}. The ability to uncover latent graph hierarchies was applied to directed acyclic graph (DAG) structures in methods like Hyperbolic Entailment Cones \cite{ganea2018}, building upon Order-Embeddings \cite{vendrov2016}, and Hyperbolic Disk embeddings \cite{suzuki2019}, with the latter achieving good performance on complex DAGs with exponentially growing numbers of ancestors and descendants. Further extensions include hyperbolic generalisations of manifold learning algorithms \cite{sala2018}, product manifolds \cite{gu2018learning}, and the inclusion of hyperbolic isometries \cite{chami2020low}.

While these methods continue to capture more complex graph topologies they are largely limited to DAGs with transitive relations, thus failing to represent many naturally occurring graphs, where cycles and non-transitive relations are common features.

In this paper we introduce pseudo-Riemannian embeddings of both DAGs and graphs with cycles. Together with a novel likelihood function with explicitly broken isometries, we are able to represent a wider suite of graph structures. In summary, the model makes the following contributions to graph representation learning, which we will expand in more detail below: 

\begin{itemize}
\item The ability to disentangle semantic and edge-based similarities using the distinction of space- and timelike separation of nodes in pseudo-Riemannian manifolds.
\item The ability to capture directed cycles by introducing a compact timelike embedding dimension. Here we consider Minkowski spacetime with a circle time dimension and anti-de Sitter spacetime.
\item The ability to represent chains in a directed graph that flexibly violate local transitivity. We achieve this by way of a novel edge probability function that decays, asymmetrically, into the past and future timelike directions.
\end{itemize}

We illustrate the aforementioned features of our model by conducting a series of experiments on several small, simulated toy networks. Because of our emphasis on graph topology, we focus on the standard graph embedding challenge of link prediction. \textit{Link prediction} is the task of inferring the missing edges of, and often solely from, a partially observed graph \cite{nickel2015review}. Premised on the assumption that the structures of real world graphs emerge from underlying mechanistic latent models (e.g. a biological evolutionary process responsible for the growth of a protein-protein interaction network, linguistic rules informing a language graph, etc), performance on the link prediction task hinges critically on one's ability to render expressive graph representations, which pseudo-Riemannian embeddings allow for beyond existing embedding methods.

With this in mind, we highlight the quality of pseudo-Riemannian embeddings over Euclidean and hyperbolic embeddings in link prediction experiments using both synthetic protein-protein interaction (PPI) networks and the DREAM5 gold standard emulations of causal gene regulatory networks. Additionally, we show that our method has comparable performance to DAG-specific methods such as Disk Embeddings on the WordNet link prediction benchmark. Finally, we explore the ability of anti-de Sitter embeddings to further capture unique graph structures by exploiting critical features of the manifold, such as its intrinsic $S^1 \times \mathbb{R}^N$ topology for representing directed cycles of different lengths.

\subsection{Related Work}

The disadvantages of Euclidean geometry compared to Minkowski spacetime for graph representation learning was first highlighted in \citet{sun2015}. It was followed by \citet{clough2017} who explore DAG representations in Minkowski space, borrowing ideas from the theory of Causal Sets \cite{bombelli1987}. More broadly, the asymptotic equivalence between complex networks and large-scale causal structure of de Sitter spacetime was proposed and studied in \citet{krioukov2012network}. Our work is notably conceptually similar to the hyperbolic disk embedding approach \cite{suzuki2019} that embeds a set of symbolic objects with a partial order relation $\preceq$ as \textit{generalized formal disks} in a quasi-metric space $(X, d)$. A formal disk $(\mathbf{x}, r) \in X \times \mathbb{R}$ is defined by a \textit{center} $\mathbf{x} \in X$ and a \textit{radius} $r \in \mathbb{R}$.\footnote{\citet{suzuki2019} generalize the standard definition of a formal disk/ball to allow for negative radii.} Inclusion of disks defines a partial order on formal disks, which enables a natural representation of partially ordered sets as sets of formal disks.

These approaches all retain the partial-order transitivity assumption where squared distances decrease monotonically into the future and past. We relax that assumption in our work, alongside considering graphs with cycles and manifolds other than Minkowski spacetime.

Pseudo-Riemannian manifold optimization was formalized in \citet{gao2018} and specialized in \citet{law2020} to general quadric surfaces in Lorentzian manifolds, which includes anti-de Sitter spacetime as a special case. 

For the remainder of the paper, nodes are points on a manifold $\MM$, the probability of edges are functions of the node coordinates, and the challenge is to infer the optimal embeddings via pseudo-Riemannian SGD on the node coordinates. 

\section{Background}
In this section we will provide a very brief overview of the relevant topics in differential geometry.\footnote{For a longer introduction, see \citet{salamon} and \citet{isham1999modern}.}

\subsection{Riemannian Manifold Optimization}
The key difference between gradient-based optimization of smooth functions $f$ on Euclidean vs. non-Euclidean manifolds is that for the latter, the trivial isomorphsim, for any $p\in \MM$, between a manifold $\mathcal{M}$ and the tangent space $T_p\mathcal{M}$ no longer holds in general. In particular, the stochastic gradient descent (SGD) update step $p' \leftarrow p - \lambda \nabla f\rvert_p$ for learning rate $\lambda$ and gradient $\nabla f$ is generalized in two areas \cite{bonnabel}:

First, $\nabla f$ is replaced with the Riemannian gradient vector field
\begin{equation}\label{riemanngrad}
    \nabla f \rightarrow \grad f := g^{-1} df,
\end{equation}
where $g^{-1}: T^*_p\MM \rightarrow T_p\MM$ is the inverse of the positive definite \textit{metric} $g$, and $df$ the differential one-form. Second, the \textit{exponential map} $\exp_p: T_p\MM \rightarrow \MM$ generalizes the vector space addition in the update equation. For any $v_p \in T_p\MM$ the first-order Taylor expansion is
\begin{equation}\label{taylor}
    f(\exp_p(v_p)) \approx f(p) + g(\grad f\rvert_p, v_p),
\end{equation}
from which we infer that $\grad f$ defines the direction of steepest descent, i.e. the Riemannian-SGD (RSGD) update step is simply
\begin{equation}\label{expupdates}
    p' \leftarrow \exp_p(-\lambda \grad f \rvert_p).
\end{equation}
The classes of manifolds considered here all have analytic expressions for the exponential map (see below). The curves traced out by $\exp_p(tv_p)$ for $t\in \mathbb{R}$ are called \textit{geodesics} - the generalisation of straight lines to curved manifolds.

\subsection{Pseudo-Riemannian Extension}
A \textit{pseudo-Riemannian} (or, equivalently, \textit{semi-Riemannian}) manifold is a manifold where $g$ is non-degenerate but no longer positive definite. If $g$ is diagonal with $\pm 1$ entries, it is a \textit{Lorentzian} manifold. If $g$ has just one negative eigenvalue, it is commonly called a \textit{spacetime} manifold. $v_p$ is labelled \textit{timelike} if $g(v_p, v_p)$ is negative, \textit{spacelike} if positive, and \textit{lightlike} or \textit{null} if zero. 

It was first noted in \citet{gao2018} that $\grad f$ is not a guaranteed ascent direction when optimizing $f$ on pseudo-Riemannian manifolds, because its squared norm is no longer strictly positive (see eq. \eqref{taylor}). For diagonal coordinate charts one can simply perform a Wick-rotation \cite{visser2017, gao2018} to their Riemannian counterpart and apply eq. \eqref{expupdates} legitimately; in all other cases, additional steps are required to reintroduce the guarantee (see Section \ref{adssection} eq. \eqref{adsgrad}).

\subsection{Example Manifolds}
\subsubsection{Minkowski Spacetime}\label{minkowski}
Minkowski spacetime is the simplest Lorentzian manifold with metric $g = \diag(-1, 1, \dotsc, 1)$. The $N+1$ coordinate functions are $(x_0, x_1, \dotsc, x_N) \equiv (x_0, \mathbf{x})$ with $x_0$ the time coordinate. The squared distance $s^2$ between two points with coordinates $x$ and $y$ is
\begin{equation}
  s^2 = - (x_0 - y_0)^2  + \sum_{i=1}^N(x_i - y_i)^2.
\end{equation}
Because the metric is diagonal and flat, the RSGD update is made with the simple map
\begin{equation}
    \exp_p(v_p) = p + v_p,
\end{equation}
where $v_p$ is the vector with components $(v_p)^i = \delta^{ij}(df_p)_j$, with $\delta^{ij}$ is the trivial Euclidean metric.

\subsubsection{Anti-de Sitter spacetime}\label{adssection}
The $(N+1)$-dimensional anti-de Sitter spacetime \cite{adsnotes} $\adsn$ can be defined as an embedding in $(N + 2)$-dimensional Lorentzian manifold $(L, g_L)$ with two `time' coordinates. Given the canonical coordinates $(x_{-1}, x_0, x_1, \dotsc, x_N)\equiv (x_{-1}, x_0, \mathbf{x})$, $\adsn$ is the quadric surface defined by $ g_L(x, x) \equiv \langle x, x \rangle_L = -1$, or in full,
\begin{equation}
   -x_{-1}^2 - x_0^2 + x_1^2 + \dotsb + x_N^2 = -1.
\end{equation}

Another useful set of coordinates involves the polar re-parameterization $(x_{-1}, x_0) \rightarrow (r, \theta)$:
\begin{equation}\label{adspolar}
    x_{-1} = r \sin \theta, \quad x_0 = r \cos \theta,
\end{equation}
where by simple substitution
\begin{equation}\label{adsradius}
    r \equiv r(\mathbf{x}) = \biggl(1 + \sum_{i=1}^N x_i^2\biggr)^\frac{1}{2}.
\end{equation}
We define a circle time coordinate to be the arc length
\begin{equation}
t := r \theta = \begin{cases}
r \sin^{-1}(x_{-1} / r), \quad &x_{0} \geq 0, \\
r (\pi - \sin^{-1}(x_{-1} / r)), \quad &x_{0} < 0, \\
\end{cases} \label{arclengthtime}
\end{equation}
with $\mathbf{x}$-dependent periodicity $t \sim t + 2 \pi r(\mathbf{x})$.

The canonical coordinates and metric $g_L$ are not intrinsic to the manifold, so we must treat the pseudo-Riemannian gradient from eq. \eqref{riemanngrad} with the projection operator $\Pi_p: T_pL \rightarrow T_p\adsn$, defined for any $v_p\in T_pL$ to be \cite{salamon}
\begin{equation}
    \Pi_p v_p = v_p + g_L(v_p, p) p.
\end{equation}

Furthermore, as \citet{law2020} proved recently for general quadric surfaces, one can sidestep the need to perform the computationally expensive Gram-Schmidt orthogonalization by implementing a double projection, i.e. we have
\begin{equation}\label{adsgrad}
    \zeta_p = \Pi_p\Bigl(g_L^{-1}\bigl(\Pi_p(g_L^{-1}df)\bigr)\Bigr)
\end{equation}
as our guaranteed descent tangent vector.

The squared distance $s^2$ between $p, q \in \adsn$ is given by
\begin{equation}
    s^2 =
    \begin{cases}
    -|\cos^{-1}(-\langle p, q\rangle_L)|^2, \quad &-1 < \langle p, q \rangle_L \leq 1, \\
    \,(\cosh^{-1}(-\langle p, q\rangle_L))^2, \quad &\hp{-1 < \,\,}\langle p, q \rangle_L < -1, \\
    \qquad\qquad 0, \quad &\hp{-1 \leq \,\,}\langle p, q \rangle_L = -1, \\
    \qquad\quad-\pi^2, &\hp{-1 \leq \,\,} \langle p, q\rangle_L > 1,
    \end{cases}
\end{equation}
where the first three cases are the squared geodesic distances between time-, space-, and lightlike separated points respectively. For $\langle p, q\rangle_L > 1$, there are no geodesics connecting the points. However, we require a smooth loss function in $s^2$ with complete coverage for all $p, q$ pairs, so we set $s^2$ to the value at the timelike limit $\langle p, q\rangle_L = 1$.

The exponential map is given by
\begin{equation}
    \exp_p(\zeta_p) = 
    \begin{cases}
    \cos (\lVert \zeta_p \rVert) p + \sin (\lVert \zeta_p \rVert) \frac{\zeta_p}{\lVert \zeta_p \rVert}, \\
    \cosh (\lVert \zeta_p \rVert) p + \sinh (\lVert \zeta_p \rVert) \frac{\zeta_p}{\lVert \zeta_p \rVert}, \\
    p + \zeta_p, 
    \end{cases}
\end{equation}
again for the time-, space-, and lightlike $\zeta_p$ respectively, and where $\lVert \zeta_p \rVert \equiv \sqrt{|g_L(\zeta_p, \zeta_p)|}$.

\section{Pseudo-Riemannian Embedding Models} \label{methods}
In this section we describe the key components of our embedding model, starting with defining a probability function for a directed edge between any two nodes in a graph. 

\subsection{Triple Fermi-Dirac Function}
The Fermi-Dirac (FD) distribution function\footnote{The case of $\alpha=1$ is the usual parameterization.}
\begin{equation}\label{fd}
    F_{(\tau, r, \alpha)}(x) := \frac{1}{e^{(\alpha x - r) / \tau} + 1},
\end{equation}
with $x\in\mathbb{R}$ and parameters $\tau, r \geq 0$ and $0 \leq \alpha \leq 1$, is used to calculate the probabilities of undirected graph edges as functions of node embedding distances \cite{krioukov2010, nickel2017}. For general directed graphs one needs to specify a preferred direction in the embedding space.\footnote{There are situations where the probability function itself is isotropic but that the features of the graph displays something analogous to spontaneous symmetry breaking where the optimal node embeddings identifies a natural direction. This is indeed the case for tree-like graphs and hyperbolic embeddings where the direction of exponential volume growth lines up with the exponential growth in the number of nodes as one goes up the tree. However for general directed graphs, an explicit symmetry breaking term in the function is needed in our case.} This was the approach taken in \citet{suzuki2019}, using the radius coordinate, and our method follows this same principle. 

For spacetime manifolds, the time dimension is the natural coordinate, indicating the preferred direction. For two points $p, q \in \MM$,  we propose the following distribution $\mathcal{F}$ which we refer to as the \emph{Triple Fermi--Dirac} (TFD) function:
\boxedeq{tfd}{\mathcal{F}_{(\tau_1, \tau_2, \alpha, r, k)}(p, q) := k (F_1F_2F_3)^{1/3},}
where $k > 0$ is a tunable scaling factor and
\begin{align}
    F_1 &:= F_{(\tau_1, r, 1)}(s^2), \label{F1} \\
    F_2 &:= F_{(\tau_2, 0, 1)}(-\Delta t), \label{F2} \\
    F_3 &:= F_{(\tau_2, 0, \alpha)}(\Delta t), \label{F3}
\end{align}
are three FD distribution terms. $\tau_1, \tau_2, r$ and $\alpha$ are the parameters from the component FD terms \eqref{fd}, $s^2$ the squared geodesic distance between $p$ and $q$, and $\Delta t$ the time coordinate difference, which in the case of Minkowski spacetime, is simply $x_0(q) - x_0(p)$. It is easy to see that for $k \leq 1$, $\mathcal{F}(p, q)$ is a valid probability value between 0 and 1.

The motivations behind eq. \eqref{tfd} can be understood by way of a series of simple plots of $\mathcal{F}$ and its component FD terms over a range of spacetime intervals on 2D Minkowski spacetime. In Figure \ref{lossheatmaps} we fix $p$ at the origin and plot $\mathcal{F}$ as a function of $q$. Concretely we demonstrate in the remainder of this section that the TFD function combines with the pseudo-Riemannian metric structure to encode three specific inductive biases inherent in general directed graphs: 1. the co-existence of semantic and edge-based similarities, 2. the varying levels of transitivity, and 3. the presence of graph cycles.

\subsubsection{Semantic Similarity vs. Direct Neighbors}
A pair of graph vertices $x, y$ can be similar in two ways. They could define an edge, in which case they are simply neighbors, or they could have overlapping neighbor sets $N_x$ and $N_y$. However, in Riemannian manifold graph embedding models, where node distances determine edge probabilities, a high degree of this latter \textit{semantic similarity} suggests that $x$ and $y$ are in close proximity, especially if both $N_x$ and $N_y$ are themselves sparsely connected \cite{sun2015}. This can be in conflict with the absence of an edge between the vertex pair. A schematic of this structure is shown in Figure \ref{toy_example_2}A. For example, in sporting events this can occur when two members of a team may compete against a largely overlapping set of competitors, but would never be pitted against one another.

Embedding graphs in a spacetime resolves this inconsistency -- for a pair of vertices that do not share an edge there is no constraint on the Jaccard index $(N_x \cap N_y) / (N_x \cup N_y)$. This claim can be verified by examining the contribution from the first FD term $F_1$ \eqref{F1} in the TFD function. As shown in Figure \ref{lossheatmaps}D, $F_1$ is low when $x$ and $y$ are spacelike separated and high when within each other's past and future lightcones. Our claim can then be rephrased in geometric terms by stating that two spacelike separated points (hence low edge probability) can have an arbitrarily large overlap in their lightcones (and hence in their neighbor sets).

\begin{figure}[ht!] 
  \centering
  \includegraphics[width=\columnwidth]{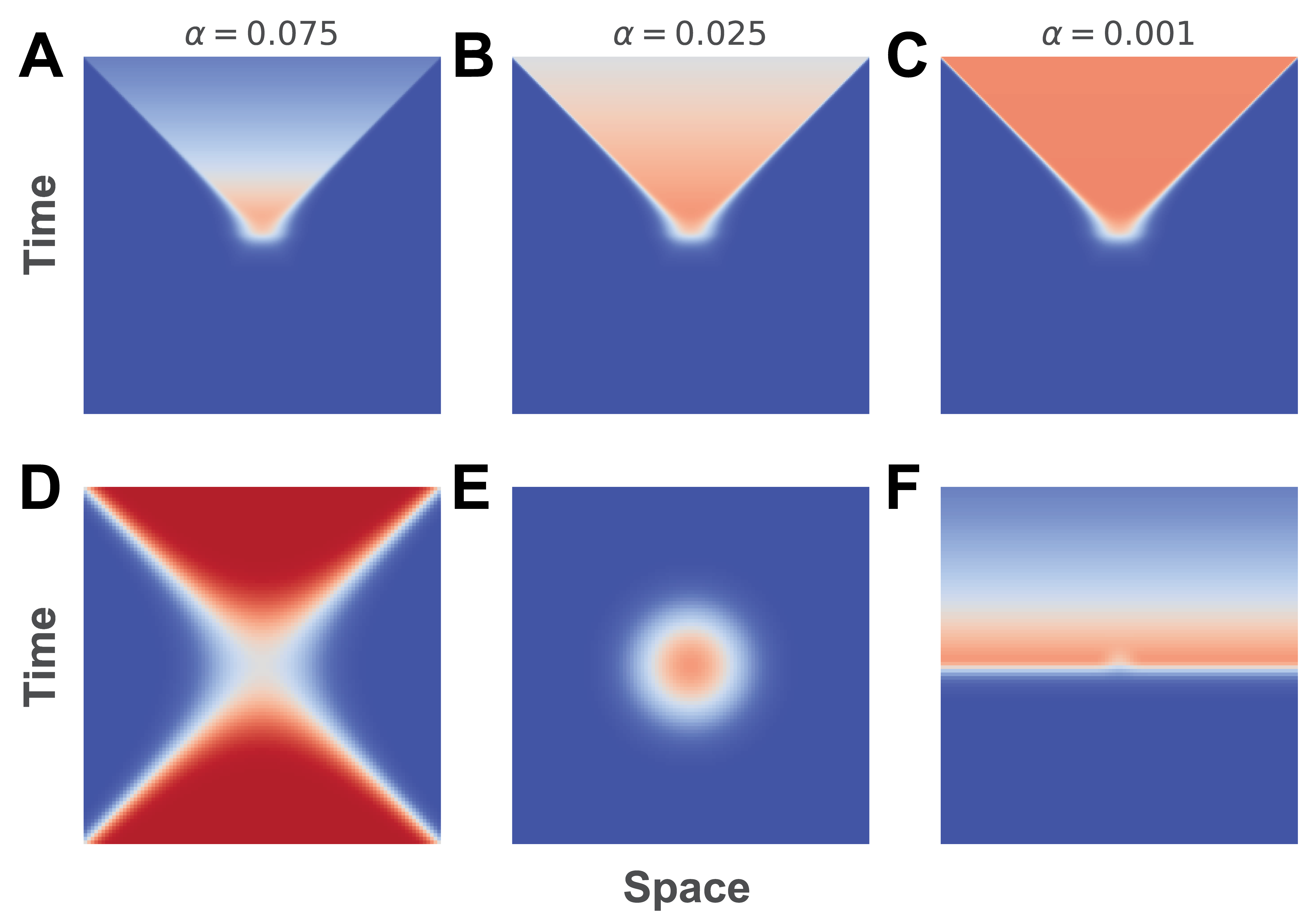}
  \caption{\textbf{A-C}: The TFD function $\mathcal{F}$ (eq. \eqref{tfd}) in 2D Minkowski spacetime  for varying $\alpha$. \textbf{D}: The FD function $F_1$ (eq. \eqref{F1}) in Minkowski space, \textbf{E}: $F_1$ in Euclidean space, \textbf{F}: $\mathcal{F}$ in Euclidean space. Blue to red corresponds probabilities from 0 to 1.}
  \label{lossheatmaps}
\end{figure}

\subsubsection{Directed Non-Transitive Relations}
Simply using $F_1$ for edge probabilities is problematic in two ways. First, the past and future are indistinguishable (Figure \ref{lossheatmaps}D), and hence too the probabilities for both edge directions of any node pair. Second, lightcones, being cones, impose a strict partial ordering on the graph. Specifically, any univariate function of $s^2$ that assigns, by way of future lightcone inclusion, a high probability to the directed chain
\begin{equation}
    p_1 \rightarrow p_2 \rightarrow p_3
\end{equation}
will always predict the transitive relation $p_1\rightarrow p_3$ with a higher probability than either $p_1\rightarrow p_2$ or $p_2 \rightarrow p_3$ (see Figure \ref{toy_example_3}A). This is a highly restrictive condition imposed by a naive graph embedding on a pseudo-Riemannian manifold. In real-world networks such as social networks, these friend-of-friend links are often missing, violating this constraint.

We tackle both the need for edge anti-symmetries and for flexible transitive closure constraints via the other two FD terms $F_2$ and $F_3$ (Eqs. \eqref{F2} and \eqref{F3}). Both terms introduce exponential decays into the respective past and future timelike directions, thereby breaking the strict partial order. Crucially, they are asymmetric due to the parameter $\alpha$ in $F_3$, which introduces different decay rates. As shown in Figure \ref{lossheatmaps}A, the TFD function has a local maximum in the near timelike future. $\alpha$ controls the level of transitivity in the inferred graph embeddings, as we see in Figures \ref{lossheatmaps}A-C. Euclidean disk embeddings \cite{suzuki2019} can be viewed as approximately equivalent to the $\alpha = 0$ case in flat Minkowski spacetime, where the region of high edge probability extends far out into the future lightcone (see Appendix \ref{diskembedding} for a precise statement of this equivalence).

We note that the TFD function is well-defined on Riemannian spaces, as long as one designates a coordinate to take the role of the time dimension (see Figure \ref{lossheatmaps}E-F).

\subsubsection{Graph Cycles}
Another feature of the TFD probabilistic model in \eqref{tfd} is that the lightcone boundaries are soft transitions that can be modified by adjusting the temperature hyperparameters $\tau_1$ and $\tau_2$. Specifically, short (in a coordinate-value sense) directed links into the past or in spacelike directions have probabilities close to $1/2$, as can be verified by a Taylor expansion of eq. \eqref{tfd} around $p=q=0$ (see Figure \ref{lossheatmaps}D).

This feature alone does not constitute a sufficient basis to promote pseudo-Riemannian embeddings with the TFD distribution as a model for cyclic graph representations (e.g. the toy example in Figure \ref{toy_example_1}A). Embedding long directed cycles in this way is not efficient. For example, a length-$n$ chain with $O(n^2)$ number of possible (and missing) transitive edges would have a similar embedding pattern to a fully-connected clique. This distinction is important when modeling real world systems, such as gene regulation, in which important information is encoded in sparse networks with cycles \cite{leclerc2008survival}. For this we turn our attention to the global topology of the manifold.

\subsection{Cylinder Topology: $S^1 \times \mathbb{R}^N$}
The problem with embedding cycles via a model that favors future-timelike directed links is the unavoidable occurrence of at least one low-probability past-directed edge. Here we propose a circular time dimension as a global topological solution. We consider two such pseudo-Riemannian manifolds with a $S^1 \times \mathbb{R}^N$ topology - a modified cylindrical Minkowski spacetime and anti-de Sitter spacetime.

\subsubsection{TFD Function on Cylindrical Minkowski Spacetime}
For an $(N+1)$-dimensional spacetime with time coordinates $(x_0, \mathbf{x})$, we construct our cylinder by identifying
\begin{equation}
    x_0 \sim x_0 + nC,
\end{equation}
for some circumference $C > 0$ and $n\in \mathbb{Z}$. To ensure a smooth TFD function on the cylinder we define a \textit{wrapped TFD} function as
\begin{equation}\label{wrappedtfd}
    \widetilde{\mathcal{F}}(p, q) := \sum_{n=-\infty}^{\infty} \mathcal{F}(p, q^{(n)}),
\end{equation}
where $x_0(q^{(n)}) \equiv x_0(q) + n C$, with all other coordinates equal. Unlike the case of the wrapped normal distribution, there is no (known) closed-form expression for $\widetilde{\mathcal{F}}$. However, provided $\alpha > 0$, the exponentially decreasing probabilities from $F_2$ and $F_3$ into the far past and future time directions respectively enables one to approximate the infinite sum in eq. \eqref{wrappedtfd} by truncating to a finite sum over integers from $-m$ to $m$. The scaling factor $k$ in eq. \eqref{tfd} can be used to ensure $\max \widetilde{\mathcal{F}} \leq 1$.

In this topology, the concept of space- and timelike separated points is purely a local feature, as there will be multiple timelike paths between \textit{any} two points on the manifold. 

\subsubsection{TFD Function on Anti-de Sitter spacetime}\label{tfdonads}
Unlike for Minkowski spacetime, AdS already has an intrinsic $S^1 \times \mathbb{R}^N$ topology.

To adapt the TFD distribution to AdS space, we adopt polar coordinates and let the angular difference between $p$ and $q$ be
\begin{equation}
    \theta_{pq} := \theta_q - \theta_p,
\end{equation}
where $\theta_p$ and $\theta_q$ correspond to the polar angle $\theta$ from eq. \eqref{adspolar}. Then we define 
\begin{equation}
    \Delta  t_{pq} := r_q \theta_{pq}.
\end{equation}

Similar to the cylindrical Minkowski case in eq. \eqref{wrappedtfd} we modify $\mathcal{F}$ such that for two points $p$ and $q$, we have
\begin{equation}
    \mathcal{F}(p, q) = \mathcal{F}(p, q^{(n)}),
\end{equation}
where $q^{(n)} \in \adsn$ have identical coordinates with $q$ except that
\begin{align}
    \Delta t_{pq}^{(n)} &:= \Delta  t_{pq} + 2 \pi n r_q.
\end{align}
The wrapped TFD distribution is identical to eq. \eqref{wrappedtfd} with $\Delta t_{pq}$ as the time difference variable $\Delta t$ in the $F_2$ and $F_3$ components (eqs. \eqref{F2} and \eqref{F3} respectively).

\section{Experiments}
\subsection{Datasets and Training Details}

In this section we evaluate the quality of pseudo-Riemannian embeddings via a series of graph reconstruction and link prediction experiments. Using small synthetic graphs, we begin with a demonstration of the model's ability to encode the particular set of graph-specific features outlined in Section \ref{methods}. We then run the model on a series of benchmarks and an ablation study to characterize the utility of these embeddings in downstream applications. For this second group of experiments, we rely on the following three classes of directed graph datasets.

\paragraph{Duplication Divergence Model:} A two-parameter model that simulates the growth and evolution of large protein-protein interaction networks \cite{ispolatov2005}. Depending on the topology of the initial seed graphs, the final graph is either a DAG or a directed graph with cycles.

\paragraph{DREAM5:}
Gold standard edges from a genome-scale network inference challenge, comprising of a set of gene regulatory networks across organisms and an \textit{in silico} example \cite{marbach2012}. These networks contain a relatively small number number of cycles.

\paragraph{WordNet:} An acyclic, hierarchical, tree-like network of nouns, each with relatively few ancestors and many descendants. We consider networks with different proportions of transitive closure. We use the same train / validation / test split as in \citet{suzuki2019} and \citet{ganea2018}.

In all our experiments, we seek to minimize the negative log-likelihood (NLL) loss based on the probabilities \eqref{fd} or \eqref{tfd}. We fix a negative sampling ratio of 4 throughout. Similar to \citet{nickel2017}, we initialize our embeddings in a small random patch near the origin ($x=(1, 0, \dotsc, 0)$ for AdS) and perform a burn-in phase of several epochs with the learning rate scaled by a factor of $0.01$. All of our models and baselines are run across various dimensions to ensure our methods do not display higher performance only at certain dimensionality level. We provide full details of each experiment, including the hyperparameter sweep and average runtimes, alongside a more full description of the datasets in the Appendix \ref{exptdetails}. The models in this paper were implemented in JAX \cite{jax2018github}.

\subsection{Graph Feature Encoding}
\begin{figure}[ht!] 
  \centering
  \includegraphics[width=\columnwidth]{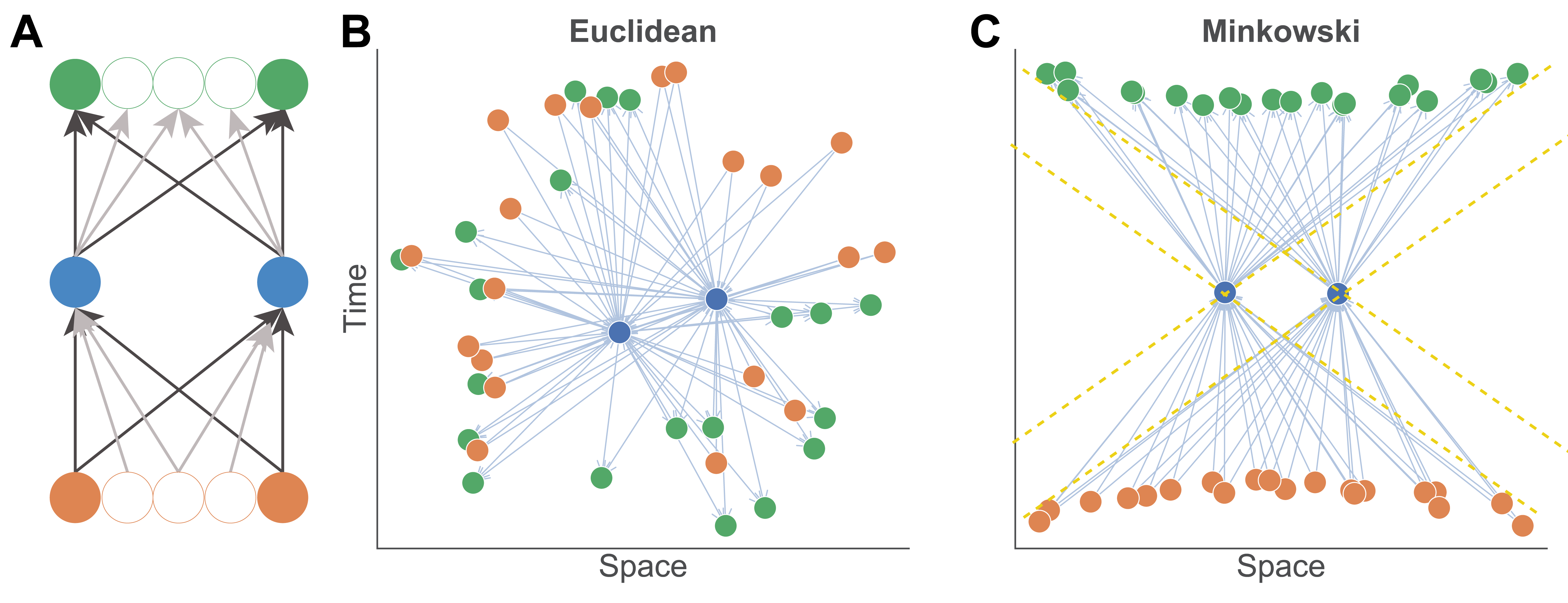}
  \caption{\textbf{A} Tri-partite graph of an unconnected node pair with large number of common predecessors and successors. \textbf{B} Euclidean embeddings with inferred edge probability between the node pair as $0.50$. \textbf{C} Minkowski embeddings with inferred edge probability of $0.01$.}
\label{toy_example_2}
\end{figure}

\begin{figure}[ht!] 
  \centering
  \includegraphics[width=\columnwidth]{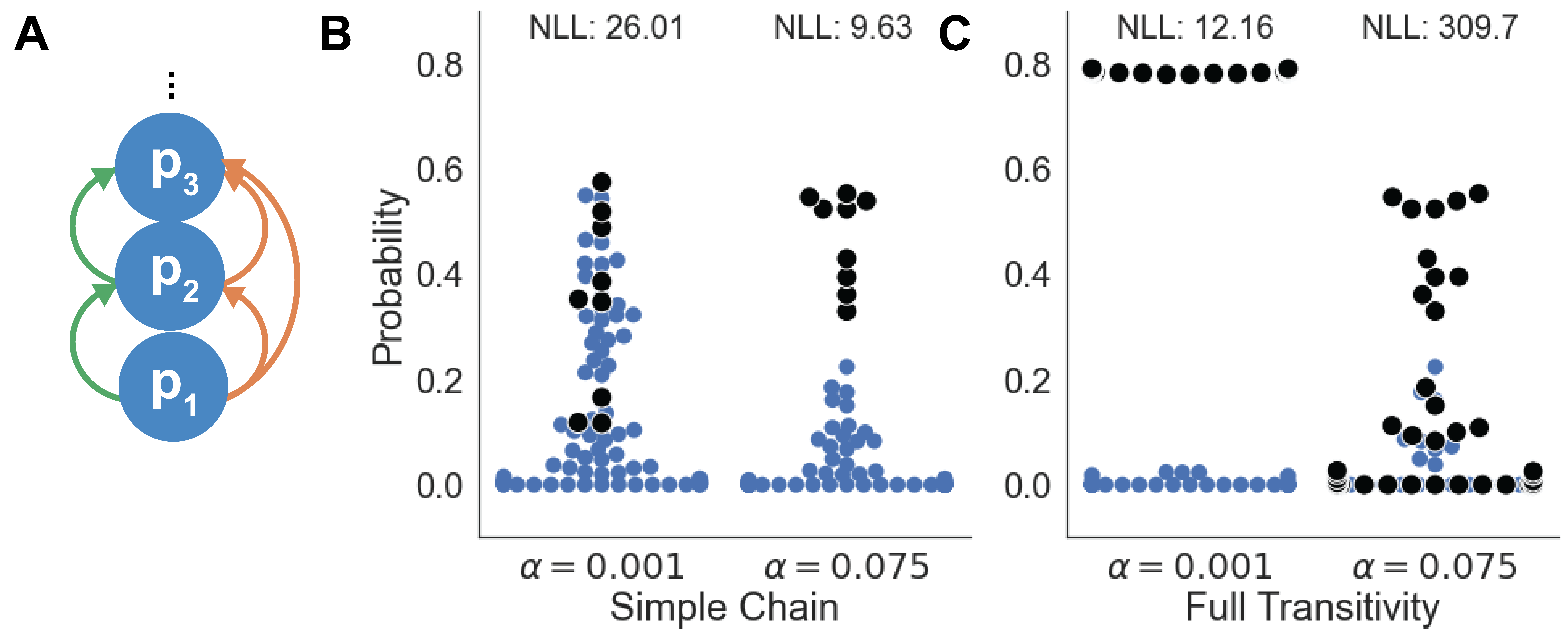}
  \caption{Performance of Minkowski embeddings for varying $\alpha$ values in the TFD function on two example graphs. \textbf{A}: Schematic of a three-node graph, with green edges indicating a simple chain, and orange edges indicating a full transitive closure. \textbf{B}: 10-node chain. \textbf{C}: 10-node fully transitive DAG. Black and blue markers indicate positive and negative directed edges respectively. Also annotated is the NLL of the graph/model pair.}
\label{toy_example_3}
\end{figure}

\begin{figure*}[ht!] 
  \includegraphics[width=\textwidth]{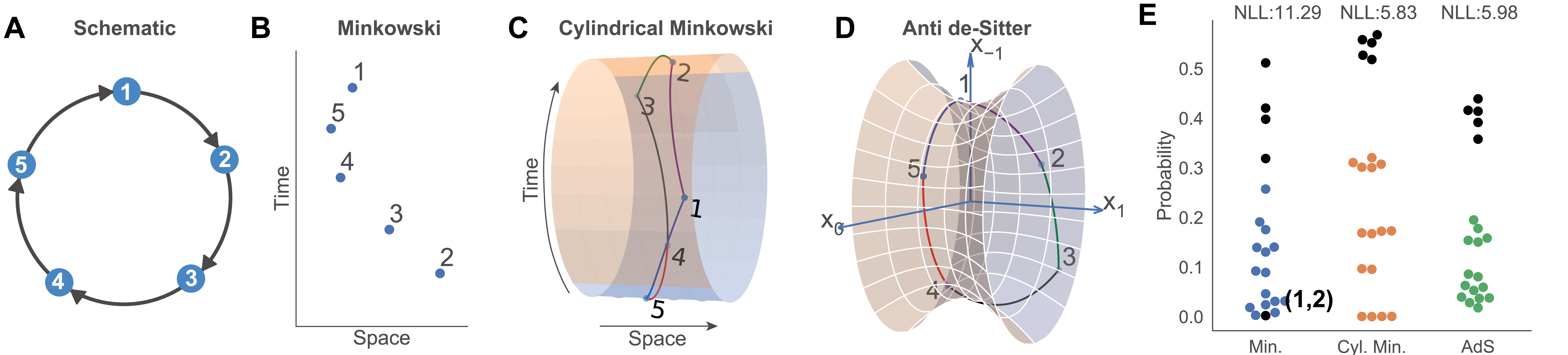}
  \caption{\textbf{A}: Illustration of a 5-node chain used as training data. \textbf{B-D}: 5-node chain embedding on various manifolds. \textbf{E}: Edge probabilities for 5-node cycle using Minkowski (Min.), cylindrical Minkowski (Cyl. Min.) and AdS embeddings. The five positive edges are indicated in black dots while the 15 negative edges are the dots with colors.}
  \label{toy_example_1}
\end{figure*}

\paragraph{Node pairs with arbitrarily large overlapping neighbor sets can remain unconnected.}
We construct a graph consisting of an unconnected node pair, and two large sets of unconnected, common predecessors and successors (Figure \ref{toy_example_2}A). The Euclidean embedding model trained with the FD edge probability function cannot reconcile the high semantic similarity of the node pair with the absence of an edge; we see that the pair of nodes are unavoidably closer to each other than the majority of their neighbor connections. (Figure \ref{toy_example_2}B). On the contrary, the Minkowski model effectively encodes the tri-partite graph as three sets of spacelike separated points with high degree of overlap in the lightcones of the nodes in the pair (Figure \ref{toy_example_2}C).

\paragraph{$\mathcal{F}$ can be tuned to encode networks with varying proportion of transitive relations}
We construct a simple, 10-node chain and a fully transitive variant to explore our ability to tune the TFD loss to best capture graph transitive structure, using the alpha parameter to vary probability decay in the time direction. We use two $\alpha$ values, $\alpha=0.001$ and $\alpha=0.075$, corresponding to the heatmaps shown in Figure \ref{lossheatmaps}A and C. In Figure \ref{toy_example_3}, we see that smaller values of $\alpha$ are able to more successfully capture full transitivity. In the case of a simple chain, setting $\alpha=0.001$ results in a higher probability being assigned to negative edges, compared to the $\alpha=0.075$ case.

\paragraph{Graph cycles wrap around the $S^1$ dimension.}
We embed a simple five-node loop (Figure \ref{toy_example_1}A) in 2D Minkowski (B), cylindrical Minkowski (C), and anti-de Sitter (D) spacetimes. In all three manifold embeddings, we recover the true node ordering as reflected in the ordered time coordinates. However for the Minkowski manifold, there is one positive edge ($1\rightarrow 2$ in this case) that unavoidably gets incorrectly assigned a very low probability, which we indicate in Figure \ref{toy_example_1}E. On the contrary, the $S^1$ time dimension for cylindrical Minkowski and AdS spacetime ensures that all edges have high probabilities thereby demonstrating the utility of the circle time dimension. We note that all the pseudo-Riemannian embedding models, including the (non-cylindrical) Minkowski case, outperform the Euclidean (NLL: 17.24) and the hyperboloid models (NLL: 13.89) (not shown in Figure \ref{toy_example_1}).

\subsection{Cyclic Graph Inference}
\begin{table*}[ht!]
\caption{Link prediction for directed cyclic graphs. We show the median average precision (AP) percentages across 20 random initializations on a held-out test set, calculated separately for different embedding dimensions $d$. Annotated in \textbf{bold} is the top-performing model for the given dimension. For reference, the asymptotic random baseline AP is $20\%$.}
\begin{adjustbox}{max width=0.999\textwidth}
\begin{tabular}{lccccccccccc}
& \multicolumn{5}{c}{Duplication Divergence Model} &  & \multicolumn{5}{c}{DREAM5 : \textit{in silico}} \\ \cline{2-6}\cline{8-12} 
  & $\hp{0}d=3\hp{0}$      & $\hp{0}d=5\hp{0}$      & $\,d=10\,$     & $\,d=50\,$     & $d=100$  & $\vphantom{\bigl(X^{X^X}}$ & $\hp{0}d=3\hp{0}$      & $\hp{0}d=5\hp{0}$      & $\,d=10\,$     & $\,d=50\,$     & $d=100$  \\ \toprule
 Euclidean + FD & 37.8    & 39.4    & 39.0    & 38.9    & 38.9 &  & 29.4 & 32.9 & 39.7 & 39.8 & 34.8  \\  
 Hyperboloid + FD            & 36.3    & 37.5    & 38.2    & 38.2    & 38.1  &  & 28.8 & 46.8 & 50.8 & 50.9 & 52.5  \\ \cmidrule{1-12}
{Minkowski + TFD}  & 43.7    & 47.5    & 48.5   & 48.5  & 48.5  &  & 36.3 & 43.1 & 51.2 & 57.7 & 58.0  \\ 
{Anti-de Sitter + TFD} & 50.1    & 52.4    & 56.2    & 56.3    & 56.8  &  & 38.1    & 45.2    & 51.9    & 55.6    & 56.0  \\
{Cylindrical Minkowski + TFD} & \textbf{55.8}   & \textbf{61.6}    & \textbf{65.3}    & \textbf{65.7}    & \textbf{65.6} &  & \textbf{41.0} & \textbf{48.4} & \textbf{56.3} & \textbf{58.9} & \textbf{61.0}  \\ 
\end{tabular}
\end{adjustbox}
\label{dupdiv_dream5_table}
\end{table*}

We perform link prediction on the Duplication Divergence Model network and the \textit{in silico} DREAM5 datasets.\footnote{For full results on DREAM5 networks for \textit{Escherichia coli} and \textit{Saccharomyces Cerevisiae} please refer to the Appendix \ref{exptdetails}. The conclusions there are similar to the main text experiment.} The results on randomly held-out positive edges for a range of embedding models of various dimensions are presented in Table \ref{dupdiv_dream5_table}. Here we see that for both datasets pseudo-Riemannian embeddings significantly outperform Euclidean and Hyperboloid embedding baselines across all dimensions, with the cylindrical Minkowski model achieving the best result. 

Next we perform an ablation study to examine the relative contributions of 1. the pseudo-Riemannian geometry, as reflected in the presence of negative squared distances, 2. the global cylindrical topology, and 3. TFD vs. FD likelihood model. The results are presented in Figure \ref{DAGnonDAGs}.

We see that all three components combine to produce the superior performances of the anti-de Sitter and cylindrical Minkowski manifold embeddings, with the $S^1$ topology arguably being the biggest contributor.

\begin{figure}[ht!] 
  \centering
  \includegraphics[width=\columnwidth]{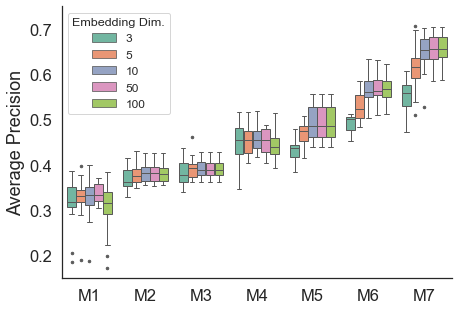}
  \caption{The average precision values of link prediction on the Duplication Divergence Model graph across manifold / likelihood model combinations and embedding dimensions, over multiple random seeds. M1: Euclidean manifold with TFD, M2: Hyperboloid + FD, M3: Euclidean with FD, M4:  cylindrical Euclidean + TFD, M5: Minkowski + TFD, M6: Anti-de Sitter + TFD, M7: cylindrical Minkowski + TFD.}
  \label{DAGnonDAGs}
\end{figure}

\subsection{DAG Inference}
We test our method on the WordNet dataset, which is a DAG that has been used in prior work on DAG inference \cite{nickel2017, ganea2018, suzuki2019}. Table \ref{Wordnet_figure} shows the performance of the pseudo-Riemannian approach on the link prediction task, benchmarked against a set of competitive methods (results taken from \citet{suzuki2019}). We break down the methods into flat- and curved-space embedding models. Our approach outperforms all of the other flat-space as well as some curved-space methods, including Poincare embeddings \cite{nickel2017}. Furthermore, the Minkowski embedding model achieves competitive performance with hyperbolic and spherical approaches. These methods are well suited to representing hierarchical relationships central to WordNet, thus this result shows that pseudo-Riemannian models are highly capable of representing hierarchies by encoding edge direction. We can show that the representational power of a special case of the TFD probability on flat Minkowski spacetime is similar to that of Euclidean Disk Embeddings (see Appendix \ref{diskembedding}). The difference in performance between Euclidean Disk Embeddings and our model on this task could be due to the additional flexibility allowed by the TFD probability function or in differences in the resulting optimization problem, something that could be further explored in future work.

\begin{table}[ht!]
\caption{F1 percentage score on the test data of WordNet. The best flat-space performance (top-half) for each dataset/embedding dimension combination has been highlighted in gray and the best overall is in \textbf{bold}. The benchmark method's results were taken from \citet{suzuki2019}. The full results table is presented in the Appendix \ref{exptdetails}.}
\label{Wordnet_figure}
\resizebox{\columnwidth}{!}{%
\begin{tabular}{lccccc}
                       & \multicolumn{2}{c}{$d=5$} &  & \multicolumn{2}{c}{$d=10$} \\ \cline{2-3} \cline{5-6} 
\textit{Transitive Closure Percentage}                       & 25\%      & 50\%      &  & 25\%       & 50\%      \\ \toprule
Minkowski + TFD (Ours) & \cellcolor[HTML]{EFEFEF}86.2    & \cellcolor[HTML]{EFEFEF}92.1    &  & \cellcolor[HTML]{EFEFEF}89.3     & \cellcolor[HTML]{EFEFEF}\textbf{94.4}    \\
Order Emb. \cite{vendrov2016}                  & 75.9    & 82.1    &  & 79.4     & 84.1   \\
Euclidean Disk \cite{suzuki2019}         & 42.5    & 45.1    &  & 65.8     & 72.0    \\ \cmidrule{1-6}
Spherical Disk \cite{suzuki2019}          & \textbf{90.5}    & \textbf{93.4}    &  & \textbf{91.5}     & 93.9    \\
Hyperbolic EC \cite{ganea2018}                    & 87.1    & 92.8    &  & 90.8     & 93.8    \\
Hyperbolic Disk \cite{suzuki2019}        & 81.3    & 83.1    &  & 90.5     & 94.2    \\
Poincare Emb. \cite{nickel2017}               & 78.3    & 83.9    &  & 82.1     & 85.4    \\ 
\end{tabular}
}
\end{table}

\section{Discussion and Future Work}
While Riemannian manifolds embeddings have made significant advances in representing complex graph topologies, there are still areas in which these methods fall short, due largely to the inherent constraints of the metric spaces used. In this paper, we demonstrated how pseudo-Riemannian embeddings can effectively address some of the open challenges in graph representation learning, summarized in the following key results. Firstly, we are able to successfully represent directed cyclic graphs using the novel Triple Fermi-Dirac probability function and the cylindrical topology of the manifolds, and to disambiguate direct neighbors from functionally similar but unlinked nodes, as demonstrated in a series of experiments on synthetic datasets. Also, Minkowski embeddings strongly outperform Riemannian baselines on a link prediction task in directed cyclic graphs, and achieve results comparable with state-of-the-art methods on DAGs across various dimensions. In this latter case, applying our approach gave us the flexibility to lift the constraints of transitivity due to the temporal decay of the TFD probability function. Using these characteristics, we demonstrate superior performance on a number of directed cyclic graphs: the duplication-divergence model and a set of three DREAM5 gold standard gene regulatory networks. 

There are two areas for further work. First, in addition to further validation of the pseudo-Riemannian optimization procedure introduced in \citet{law2020}, one can experiment with AdS coordinates different to the one proposed in Section \ref{tfdonads}, or extend the methods in this work to more general classes of pseudo-Riemannian manifolds. Second, we note that the field closely related to link prediction is network inference -- inferring causal networks from observational and interventional data is a critical problem with a wide range of applications. In biology, for example, robust and accurate inference of gene regulatory networks from gene expression data is a long-standing grand challenge in systems biology. In this case, limited ground-truth connectivity is usually already available as a prior -- a logical extension of the model presented here would be a hybrid model that combines pseudo Riemannian-based link prediction with metric space network inference.

\section*{Acknowledgements}
We would like to thank Craig Glastonbury for helpful discussions and suggestions on the manuscript.

\bibliography{references}
\bibliographystyle{icml2021}
\newpage
\appendix

\section{Relationship with Disk Embeddings}\label{diskembedding}
We prove the following result, relating Euclidean Disk Embeddings \cite{suzuki2019} to the Triple Fermi-Dirac (TFP) model in Minkowski spacetime.

\begin{proposition}
Let $p=(\mathbf{x},u)$ and $q=(\mathbf{y},v)$ denote elements of $\mathbb{R}^{n} \times \mathbb{R}$, $D$ be the Euclidean distance between $\mathbf{x}$ and $\mathbf{y}$, and $T=v-u$ the time difference. Then, for the choice of TFD function parameters $\alpha = 0, r=0, k=1$ we have $\mathcal{F}(p,q) \geq \frac{1}{2}$ if and only if
\begin{equation}\label{eq:minkowski_pos}
D \leq \left(\tau_1\log\frac{3 - e^{-T / \tau_2}}{1 + e^{-T / \tau_2}} + T^2\right)^{1/2}
\end{equation}
\end{proposition}
\begin{proof}
From the definition of the Triple Fermi-Dirac probability on Minkowski space-time we have 
\begin{equation*}
    \mathcal{F}(p,q) = \left(\frac{1}{e^{(D^2-T^2)/\tau_1} + 1}\frac{1}{e^{-T/\tau_2}+1}\frac{1}{2}\right)^{1/3}
\end{equation*}
Then,
\begin{align*}
    \mathcal{F}(p,q) \geq \frac{1}{2} 
    &\Longleftrightarrow \frac{1}{e^{(D^2-T^2)/\tau_1} + 1}\frac{1}{e^{-T/\tau_2}+1} \geq \frac{1}{4} \\ 
    &\Longleftrightarrow e^{(D^2 - T^2)/\tau_1} \leq \frac{3 - e^{-T/\tau_2}}{1 + e^{-T / \tau_2}} \\
    &\Longleftrightarrow D \leq \left(\tau_1\log\frac{3 - e^{-T / \tau_2}}{1 + e^{-T / \tau_2}} + T^2\right)^{1/2} \qedhere
\end{align*}
\end{proof}

The set of points $p$ and $q$ from $\mathbb{R}^n \times \mathbb{R}$ that satisfies the condition for inclusion of Euclidean Disks, which determines the embeddings that are connected by directed edges, is given by $D \leq T$. For $T \geq 0$, every pair $p$ and $q$ which satisfy $D \leq T$ also satisfy (\ref{eq:minkowski_pos}); hence the set of points $p,q \in \mathbb{R}^n \times \mathbb{R}$ which correspond to a directed edge in the Euclidean Disk Embeddings model is strictly contained in the set of pairs of points which have $\mathcal{F}(p,q) \geq \frac{1}{2}$ in the Triple Fermi-Dirac probability function on (flat) Minkowski space-time. Moreover, the difference between these two sets is a small neighbourhood of the Minkowski light cone, with the size of this set dependent on the parameters $\tau_1$ and $\tau_2$. Figure~\ref{fig:de_cone_comparison} illustrates this in the $\mathbb{R} \times \mathbb{R}$ case for $\tau_1 = \tau_2 = 0.05$ (the parameters used by the model in the experiments on the WordNet dataset.)

\begin{figure}[ht!] 
\label{fig:de_cone_comparison}
  \centering
  \caption{Boundaries of the regions containing points $q \in \mathbb{R} \times \mathbb{R}$ with $\mathcal{F}(p,q) \geq \frac{1}{2}$ (red) and with $D \leq T$ (blue) for $p = (0,0)$}
  \includegraphics[width=\columnwidth]{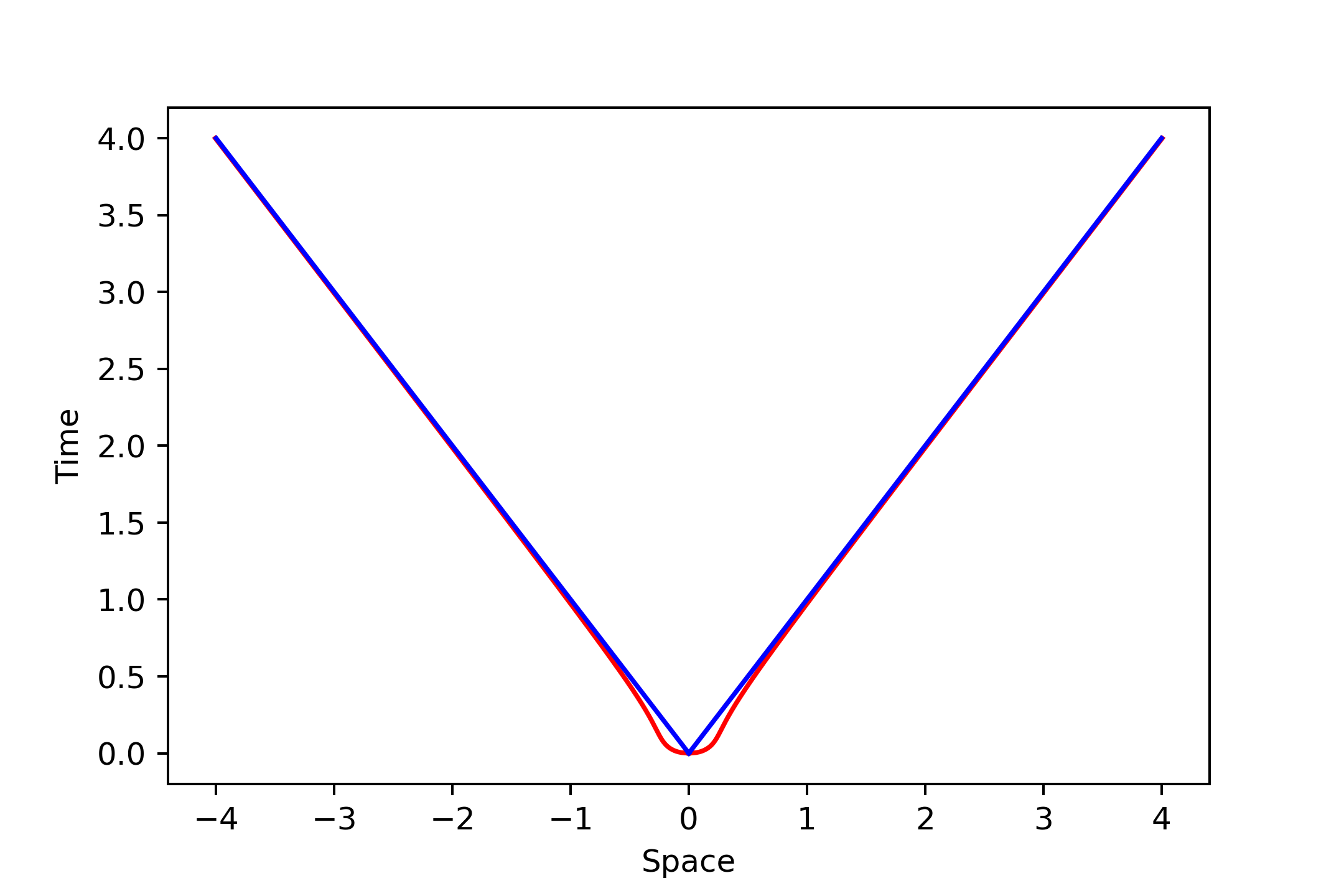}
\end{figure}

\section{Experiment details}\label{exptdetails}
\subsection{Graph Datasets}
\begin{table}[]
\label{dataset_statistics}
\caption{Statistics of datasets used in the experiments.}
\resizebox{\columnwidth}{!}{%
\begin{tabular}{lllc}
\hline
Dataset                & Nodes  & Edges   & Cyclic \\ \hline
Duplication Divergence & 100      & 1026       & True   \\
DREAM5: \textit{In silico}      & 1,565  & 4,012   & True   \\
DREAM5: \textit{E. coli}        & 1,081  & 2,066   & True   \\
DREAM5: \textit{S. cerevisiae}  & 1,994  & 3,940   & True   \\
WordNet                & 82,115 & 743,086 & False  \\ \hline
\end{tabular}%
}
\end{table}

This section describes in detail the experiments conducted in our work, including dataset and evaluation specifications.

\subsection{Hyperparameters and training details}
The model hyperparameters for the experiments on the Duplication Divergence Model and DREAM5 datasets are given in Table \ref{hptable}.  We performed a Cartesian product grid search and selected the optimal values based on the median average precision on the test set over three random initializations. We implemented a simple linear learning rate decay schedule to a final rate of a quarter of the initial rate $\lambda$; hence although we selected the best test set performance across all the training run epochs, the maximum number of epochs affects the performance as it determines the rate of learning-rate decay.

The hyperparameter tuning approach for the WordNet experiments was performed differently as it includes an F1 score threshold selection. This is described separately in Appendix \ref{wordnet_supp} below.

A summary of the datasets is provided in Table \ref{dataset_statistics}.

\subsubsection{Duplication Divergence Model}
The duplication divergence model is a simple two-parameter model of the evolution of protein-protein interaction networks \cite{ispolatov2005}. Starting from a small, fully-connected, directed graph of $n_\mathrm{i}$ nodes, the network is grown to size $n_\mathrm{f}$ as follows. First we duplicate a random node from the set of existing nodes. Next the duplicated node is linked to the neighbors of the original node with probability $p_1$ and to the original itself with probability $p_2$. This step is repeated $(n_\mathrm{f}-n_\mathrm{i})$ times. Because the model only includes duplicated edges with no extra random attachments, it is easily shown that if the seed network is a directed acyclic graph (DAG), the network remains a DAG throughout its evolution. We generate one cyclic graph from an initial 3-node seed graph, with parameters $(n_\mathrm{i}, n_\mathrm{f}, p_1, p_2) = (3, 100, 0.7, 0.7)$ and perform a single random 85 / 15 train-test split for all experiments.

\subsubsection{DREAM5}
Traditionally a network inference problem, the DREAM5 challenge \cite{marbach2012} is a set of tasks with corresponding datasets. The goal of the challenge and each task is to infer genome-scale transcriptional regulatory networks from gene-expression microarray datasets. We use only the gold standard edges from three DREAM5 gene regulatory networks for our use case: \textit{In silico}, \textit{Escherichia coli} and \textit{Saccharomyces cerevisiae}. The fourth \textit{Staphylococcus aureus} network from the DREAM5 challenge was not considered here due to the issues in evaluation described in \citet{marbach2012}. We extract only the positive-regulatory nodes from the remaining three networks while omitting the gene-expression data itself. We limit ourselves to positive-regulation edges as our work focuses on single edge type graphs. The extracted node pairs form a directed graph, where the edge between nodes $(i, j)$ represent a regulatory relation between two genes. Each network is then randomly split into train and test sets, following 85 / 15 split. Although the network is cyclic, the number of cycles in each of them is relatively low, which may account for the similarity in performances of the (non-cylindrical) Minkowski embeddings with embeddings on cylindrical Minkowski and AdS (e.g. the \textit{E. Coli} dataset experiments in Table \ref{dupdiv_dream5_table_results}). Similar to the Duplication Divergence model above, we evaluate the performance using the average precision on the test set.

\subsubsection{WordNet}
\label{wordnet_supp}
To further evaluate the embeddings of directed acyclic graphs (DAGs), we use the \textit{WordNet} (Miller, 1998) database of noun hierarchy. The WordNet dataset is an example of a tree-like network with low number of ancestors, and high number of descendants. To ensure fair comparison we use the same dataset and split as in \citet{suzuki2019}, proposed in \citet{ganea2018}. During training and evaluation, we use nodes $(i, j) \in \mathcal{T} $ connected by an edge such that $c_{i} \preceq c_{j}$ as positive examples. As the dataset consists only of positive examples, we randomly sampled a set of non-connected negative pairs for each positive pair. For evaluation, we use an F1 score for a binary classification to assess whether a given pair of nodes $(i, j)$ is connected by a directed edge in the graph. Similarly as in (Suzuki, 2019) we compute the results based on the percentage of transitive closure of the graph.

As WordNet is a DAG, we restrict our attention to the model based on flat Minkowski spacetime. Here, the key hyper-parameters are the learning rate and values for $\alpha$, $\tau_1$ and $\tau_2$ (we set $r = 0$ and $k=1$ in these experiments). We perform a Cartesian product grid search and select the optimal values based on the mean F1 score on the validation set over two random initializations (seeds), using $50$ epochs and batch size $50$. The threshold for the F1 score was chosen by line search and tuned.  The grid of values is as follows, with the optimal values given in \textbf{bold}: $\lambda \in \{0.01, \mathbf{0.02}, 0.1\}$, $\alpha \in \{0., 0.0075, \mathbf{0.075}\}$, $\tau_1 \in \{0.005, \mathbf{0.05}, 1.\}$ and $\tau_2 \in \{0.005, \mathbf{0.05}, 1.\}$. We then trained a model with the selected parameters for $150$ epochs using $5$ different random seeds to generate the results reported in Table~\ref{Wordnet_table}. Here, we report the test set F1 score using a threshold selected via line search on the validation set.

\begin{table*}[ht!]
\caption{F1 percentage score on the test data on WordNet. The best flat-space performance (top-half) for each dataset/embedding dimension combination has its background highlighted in gray, and the best overall is highlighted in \textbf{bold}. The benchmark methods results were taken from \cite{suzuki2019}. For the results of our method, we report the median together with standard deviation across seeds (N=5).}
\label{Wordnet_table}
\begin{adjustbox}{max width=\textwidth}
\begin{tabular}{lccccccccc}
                        & \multicolumn{4}{c}{$d=5$} &  & \multicolumn{4}{c}{$d=10$} \\ 
                        \cline{2-5}  \cline{7-10} 
\textit{Transitive Closure Percentage}  & 0\%    & 10\%     & 25\%      & 50\%  &    & 0\%    & 10\%      & 25\%       & 50\%      \\ \toprule
Minkowski + TFD (Ours)                  & 21.2 ± 0.5    & \cellcolor[HTML]{EFEFEF}77.8 ± 0.1    
                                        & \cellcolor[HTML]{EFEFEF}86.2 ± 0.7   & \cellcolor[HTML]{EFEFEF}92.1 ± 0.3  & 
                                        & 24 ± 0.5    & \cellcolor[HTML]{EFEFEF}82 ± 1.2    
                                        & \cellcolor[HTML]{EFEFEF}89.3 ± 0.7  & \cellcolor[HTML]{EFEFEF}\textbf{94.4  ± 0.1}    \\
Order Emb. \cite{vendrov2016}           & 34.4  & 70.6      & 75.9    & 82.1    &    
                                        & 43.0     & 69.7    & 79.4     & 84.1   \\
Euclidean Disk \cite{suzuki2019}        & \cellcolor[HTML]{EFEFEF}35.6  & 38.9      & 42.5    & 45.1            &  
                                        & \cellcolor[HTML]{EFEFEF}\textbf{45.6}  & 54.0   & 65.8    & 72.0  \\ \hline
Spherical Disk \cite{suzuki2019}        & \textbf{37.5}  & \textbf{84.8}  & \textbf{90.5} & \textbf{93.4} &  
                                        & 42.0           & \textbf{86.4} & \textbf{91.5} & 93.9    \\
Hyperbolic Disk \cite{suzuki2019}       & 32.9  & 69.1  & 81.3          & 83.1          &  
                                        & 36.5   & 79.7          & 90.5          & 94.2    \\
Hyperbolic EC \cite{ganea2018}          & 29.2  & 80.0  & 87.1          & 92.8          &  
                                        & 32.4   & 84.9          & 90.8          & 93.8    \\
Poincare Emb. \cite{nickel2014}         & 28.1  & 69.4  & 78.3          & 83.9          &  
                                        & 29.0   & 71.5          & 82.1          & 85.4    \\ 
\end{tabular}
\end{adjustbox}
\end{table*}

\begin{table*}[ht!]
\label{dupdiv_dream5_table_results}
\caption{Link prediction for directed cyclic graphs with embedding dimension $d$. Reported above are the median average precision (AP) percentages with standard deviation across seeds (N=20), calculated on a held-out test set for varying embedding dimension. Annotated in \textbf{bold} is the top-performing model for the given dimension. For reference, the random baseline AP is $20\%$.} 
\resizebox{\textwidth}{!}{%
\begin{tabular}{llccccc}
 &                             & \multicolumn{5}{c}{Embedding dimension}                        \\ \cline{3-7} 
 &                             & 3          & 5          & 10         & 50         & 100        \\ \toprule
\multirow{5}{*}{Duplication Divergence} & Euclidean + FD & $37.8 \pm 2.8$    & $39.4 \pm 2.4$    & $39.0 \pm 1.9$    & $38.9 \pm 1.9$    & $38.9 \pm 1.9$    \\ 
 & Hyperboloid + FD            & $36.3 \pm 2.2$    & $37.5 \pm 2.4$    & $38.2 \pm 2.3$    & $38.2 \pm 2.4$    & $38.1 \pm 2.3$    \\ \cmidrule{2-7} 
 & Minkowski + TFD             & $43.7 \pm 2.2$    & $47.5 \pm 2.5$    & $48.5 \pm 3.7$    & $48.5 \pm 3.7$    & $48.5 \pm 3.7$    \\ 
 & Anti de-Sitter + TFD & $50.1 \pm 3.2$    & $52.4 \pm 3.3$    & $56.2 \pm 3.2$    & $56.3 \pm 3.1$    & $56.8 \pm 3.0$    \\ 
 & Cylindrical Minkowski + TFD & $\mathbf{55.8 \pm 3.6}$    & $\mathbf{61.6 \pm 4.8}$    & $\mathbf{65.3 \pm 4.1}$    & $\mathbf{65.7 \pm 3.1}$    & $\mathbf{65.6 \pm 3.2}$    \\ \midrule
\multirow{5}{*}{DREAM5: \textit{in silico}}      & Euclidean + FD & $29.4 \pm 2.1$ & $32.9 \pm 2.5$ & $39.7\pm 1.8$  & $39.8 \pm 1.6$ & $34.8\pm 1.1$  \\
 & Hyperboloid + FD            & $28.8 \pm 5.5$ & $46.8 \pm 4.6$ & $50.8 \pm 7.4$ & $50.9 \pm 1.5$ & $52.5 \pm 1.5$ \\ \cmidrule{2-7} 
 & Minkowski + TFD             & $36.3 \pm 2.3$ & $43.1 \pm 3.1$ & $51.2 \pm 3.0$ & $57.7 \pm 2.8$ & $58.0 \pm 2.7$ \\ 
 & Anti de-Sitter + TFD        & $38.1 \pm 4.8$    & $45.2 \pm 2.3 $   & $51.9 \pm 5.2$    & $55.6 \pm 4.2$    & $56.0 \pm 3.4$    \\ 
 & Cylindrical Minkowski + TFD & $\mathbf{41.0 \pm 3.6}$ & $\mathbf{48.4 \pm 7.3}$ & $\mathbf{56.3 \pm 8.4}$ & $\mathbf{58.9 \pm 2.9}$ & $\mathbf{61.0 \pm 1.9}$ \\  \midrule

\multirow{5}{*}{DREAM5: \textit{E. Coli}}        & Euclidean + FD & $33.0 \pm 3.9$ & $34.2\pm 3.4$  & $40.2 \pm 4.3$ & $44.5 \pm 2.6$ & $49.0 \pm 3.2$ \\
 & Hyperboloid + FD            & $43.4 \pm 4.1$ & $47.2 \pm 3.3$ & $52.7 \pm 1.9$ & $53.6 \pm 1.4$ & $50.6 \pm 0.7$ \\ \cmidrule{2-7} 
 & Minkowski + TFD             & $\mathbf{51.0 \pm 4.0}$ & $\mathbf{58.4 \pm 2.3}$ & $\mathbf{63.4 \pm 3.6}$ & $\mathbf{67.7 \pm 2.7}$ & $\mathbf{68.2 \pm 2.4}$ \\ 
 & Anti de-Sitter + TFD        & $42.7 \pm 3.7$    & $56.5 \pm 2.6$    & $61.8 \pm 6.8$    & $63.3 \pm 4.8$    & $63.0 \pm 7.5$    \\  
 & Cylindrical Minkowski + TFD & $50.3 \pm 3.3$ & $56.8 \pm 3.4$ & $62.3 \pm 3.3$ & $65.8 \pm 3.4$ & $63.2 \pm 2.4$ \\ \midrule
\multirow{5}{*}{DREAM5: \textit{S. Cerevisiae}}  & Euclidean + FD & $33.0 \pm 2.7$ & $34.2 \pm 2.8$ & $40.2 \pm 3.3$ & $44.5 \pm 3.5$ & $49.0 \pm 2.0$ \\
 & Hyperboloid + FD            & $29.2 \pm 2.5$ & $37.9 \pm 1.3$ & $46.5 \pm 1.6$ & $48.8 \pm 1.4$ & $47.9 \pm 1.2$ \\ \cmidrule{2-7} 
 & Minkowski + TFD             & $34.7 \pm 2.2$ & $38.6 \pm 1.9$ & $46.4 \pm 3.1$ & $52.7 \pm 3.0$ & $54.0 \pm 2.5$ \\ 
 & Anti de-Sitter + TFD        & $37.2 \pm 3.2$    & $41.3 \pm 1.5$    & $44.9. \pm 2.5 $   & $47.5 \pm 3.1$    & $49.4 \pm 3.3$    \\ 
 & Cylindrical Minkowski + TFD & $\mathbf{37.4 \pm 3.2}$ & $\mathbf{42.7 \pm 2.3}$ & $\mathbf{46.8 \pm 3.5}$ & $\mathbf{53.4 \pm 2.2}$ & $\mathbf{54.6 \pm 2.1}$ \\
\end{tabular}%
}
\end{table*}

\begin{table*}[ht!]
\caption{Cyclic graph inference hyperparameters. Optimal values* are in \textbf{bold}. * In a situation where the optimal value was dependent on the embedding dimension, we highlight all the values that recorded best performance on at least one embedding dimension.  }
\begin{adjustbox}{max width=1.3\textwidth, angle=90}
\begin{tabular}{l|ccccccccc}
& circumference & $\tau_1$ & $\tau_2$ & $r$ & $\lambda$ & batch size & $\alpha$ & Max epochs \\
\toprule
DD Model (Euclidean) & - &$(0.04, 0.075, 0.15, \mathbf{0.4}, 1.0, 2.0)$ & - & $(\mathbf{0.0}, 0.1, 1.0)$ & $(0.08, \mathbf{0.02}, 0.001)$ & $\mathbf{4}$ & - & $\mathbf{300}$ \\
DD Model (Hyperboloid) & - &$(0.04, \mathbf{0.075}, 0.15, 0.4, 1.0, 2.0)$ & - & $(\mathbf{0.0}, 0.1, 1.0)$ & $(0.08, 0.02, \mathbf{0.001})$ &  $\mathbf{4}$ & - & $\mathbf{300}$ \\
DD Model (Minkowski) & - &$(\mathbf{0.075}, 0.15, 0.4)$ & $(\mathbf{0.03}, 0.05, 0.07)$ & $\mathbf{0.0}$ & $(0.08, \mathbf{0.02})$ & $\mathbf{2}$ & $(0.03, 0.045, \mathbf{0.06}, 0.075, 0.09)$ & $\mathbf{200}$ \\
DD Model (Cylindrical Minkowski) & $(6, 8, \mathbf{10}, 12, 14)$ & $(0.075, 0.15, \mathbf{0.4})$ & $(0.03, 0.05, \mathbf{0.07})$ & $\mathbf{0.0}$ & $(0.08, \mathbf{0.02})$ & $\mathbf{2}$ & $(0.03, 0.045, 0.06, 0.075, \mathbf{0.09})$ & $\mathbf{200}$ \\
DD Model (AdS) & - &$\mathbf{0.4}$ & $\mathbf{0.15}$ & $(-0.25, -0.2, -0.15, \mathbf{-0.1}, -0.05, -0.025, 0.0)$ & $(0.02, \mathbf{0.016}, 0.012, 0.008)$ & $\mathbf{2}$ & $\mathbf{0.15}$ & $\mathbf{150}$ \\
\midrule
DREAM5: In Silico (Euclidean) & - &$(\mathbf{0.04}, 0.075, 0.15, 0.4, 1.0, 2.0)$ & - & $(\mathbf{0.0}, 0.1, 0.5)$ & $(0.2, 0.08, \mathbf{0.02}, \mathbf{0.002})$ & $\mathbf{2}$ & - & $\mathbf{60}$ \\
DREAM5: In Silico (Hyperboloid) & - &$(0.04, 0.075, \mathbf{0.15}, 0.4, 1.0, 2.0)$ & - & $(\mathbf{0.0}, 0.1, \mathbf{0.5})$ & $(0.2, \mathbf{0.08}, 0.02, \mathbf{0.002})$ &  $\mathbf{2}$ & - & $\mathbf{60}$ \\
DREAM5: In Silico (Minkowski) & - &$(0.075, \mathbf{0.15}, 0.4)$ & $(0.03, \mathbf{0.05}, 0.07)$ & $\mathbf{0.0}$ & $(\mathbf{0.2}, 0.08, \mathbf{0.02})$ & $\mathbf{2}$ & $(0.03, 0.045, \mathbf{0.06}, 0.075, 0.09)$ & $\mathbf{60}$ \\
DREAM5: In Silico (Cylindrical Minkowski) & $(6, \mathbf{8}, 10, 12, 14)$ & $(0.075, \mathbf{0.15}, 0.4)$ & $(\mathbf{0.03}, 0.05, 0.07)$ & $\mathbf{0.0}$ & $(\mathbf{0.2}, 0.08, \mathbf{0.02})$ & $\mathbf{2}$ & $(0.03, 0.045, 0.06, \mathbf{0.075}, 0.09)$ & $\mathbf{60}$ \\
DREAM5: In Silico (AdS) & - &$0.075, \mathbf{0.15}$ & $(0.05, \mathbf{ 0.07})$ & $\mathbf{-0.1}$ & $(\mathbf{0.01}, \mathbf{0.02})$ & $\mathbf{2}$ & $(0.03, 0.04, 0.05, \mathbf{0.06}, 0.07)$ & $\mathbf{60}$ \\
\midrule
DREAM5: E. Coli (Euclidean) & - &$(0.04, 0.075, 0.15, 0.4, \mathbf{1.0}, 2.0)$ & - & $(\mathbf{0.0}, 0.1, 0.5)$ & $(\mathbf{0.08}, \mathbf{0.02}, 0.002)$ & $\mathbf{2}$ & - & $\mathbf{60}$ \\
DREAM5: E. Coli (Hyperboloid) & - &$(0.04, 0.075, \mathbf{0.15}, 0.4, 1.0, 2.0)$ & - & $(\mathbf{0.0}, \mathbf{0.1}, 0.5)$ & $(0.08, 0.02, \mathbf{0.002})$ &  $\mathbf{2}$ & - & $\mathbf{60}$ \\
DREAM5: E. Coli (Minkowski) & - &$(0.075, \mathbf{0.15}, 0.4)$ & $(0.03, 0.05, \mathbf{0.07})$ & $\mathbf{0.0}$ & $(\mathbf{0.2}, 0.08, \mathbf{0.02})$ & $\mathbf{2}$ & $(0.03, 0.045, \mathbf{0.06}, 0.075, 0.09)$ & $\mathbf{60}$ \\
DREAM5: E. Coli (Cylindrical Minkowski) & $(6, \mathbf{8}, 10, \mathbf{12}, 14)$ & $(\mathbf{0.075}, 0.15, 0.4)$ & $(0.03, \mathbf{0.05}, 0.07)$ & $\mathbf{0.0}$ & $(0.2, 0.08, \mathbf{0.02})$ & $\mathbf{2}$ & $(0.06, \mathbf{0.075}, 0.09)$ & $\mathbf{60}$ \\
DREAM5: E. Coli (AdS) & - &$\mathbf{0.15}$ & $\mathbf{0.07}$ & $\mathbf{-0.1}$ & $(\mathbf{0.01}, \mathbf{0.02})$ & $\mathbf{2}$ & $\mathbf{0.06}$ & $\mathbf{60}$ \\
\midrule
DREAM5: S. Cerevisiae (Euclidean) & - &$(0.04, 0.075, 0.15, 0.4, 1.0, \mathbf{2.0})$ & - & $(\mathbf{0.0}, \mathbf{0.1}, 0.5)$ & $(\mathbf{0.2}, , \mathbf{0.08}, 0.02, 0.002)$ & $\mathbf{2}$ & - & $\mathbf{60}$ \\
DREAM5: S. Cerevisiae (Hyperboloid) & - &$(0.04, 0.075, 0.15, 0.4, 1.0, \mathbf{2.0})$ & - & $(\mathbf{0.0}, 0.1, 0.5)$ & $(\mathbf{0.2}, \mathbf{0.08}, 0.02, 0.002)$ &  $\mathbf{2}$ & - & $\mathbf{60}$ \\
DREAM5: S. Cerevisiae (Minkowski) & - &$(0.075, \mathbf{0.15}, 0.4)$ & $(\mathbf{0.03}, 0.05, 0.07)$ & $\mathbf{0.0}$ & $(\mathbf{0.2}, 0.08, \mathbf{0.02})$ & $\mathbf{2}$ & $(0.03, \mathbf{0.045}, 0.06, 0.075, 0.09)$ & $\mathbf{60}$ \\
DREAM5: S. Cerevisiae (Cylindrical Minkowski) & $(6, \mathbf{8}, 10, \mathbf{12}, 14)$ & $(\mathbf{0.075}, 0.15, 0.4)$ & $(0.03, \mathbf{0.05}, 0.07)$ & $\mathbf{0.0}$ & $(0.2, 0.08, \mathbf{0.02})$ & $\mathbf{2}$ & $(0.03, 0.045, 0.06, \mathbf{0.075}, 0.09)$ & $\mathbf{60}$ \\
DREAM5: S. Cerevisiae (AdS) & - &$\mathbf{0.15}$ & $\mathbf{0.07}$ & $\mathbf{-0.1}$ & $(\mathbf{0.01}, \mathbf{0.02})$ & $\mathbf{2}$ & $\mathbf{0.06}$ & $\mathbf{60}$ \\
\bottomrule
\end{tabular}
\end{adjustbox}
\label{hptable}
\end{table*}

\subsection{Model complexity and runtimes}
The model complexity is similar to standard (Euclidean) embedding models. For the pseudo-Riemannian manifold examples in this paper, the computation of the pseudo-Riemannian SGD vector from the JAX autodiff-computed value for the differential $df$ scales linearly with embedding dimension. The approximation to infinite sum in the wrapped TFD function introduces a $\mathcal{O}(m)$ complexity factor, where $m$ is the number of cycles one uses in the approximation (see eq. (21) in the main text).

The cyclic graph link prediction experiments were performed as CPU-only single process runs. The longest runtime for the 100-dimensional AdS manifold embedding model was $\sim 19$ $\text{sec} / \text{epoch}$ for the Duplication Divergence model graph and $\sim 112$ $\text{sec} / \text{epoch}$ for the \textit{in silico} DREAM5 dataset. The equivalent runtimes for the other manifolds are approximately an order of magnitude shorter.

The WordNet experiments were performed on NVIDIA V100 GPUs. The run-times were largely similar for the cases with embedding dimension $5$ or $10$, with the proportion of the transitive closure included in the training data (i.e. the training data size) being the main factor which determined run-time in our experiments. These varied from $\sim 60$ $\text{sec} / \text{epoch}$ when $0\%$ was included to $\sim 280$  $\text{sec} / \text{epoch}$ when $50\%$ was included. 

\end{document}